\documentclass[sn-basic,Numbered, iicol]{sn-jnl}

\usepackage{graphicx}%
\usepackage{multirow}%
\usepackage{amsmath,amssymb,amsfonts}%
\usepackage{amsthm}%
\usepackage{mathrsfs}%
\usepackage[title]{appendix}%
\usepackage{xcolor}%
\usepackage{textcomp}%
\usepackage{manyfoot}%
\usepackage{booktabs}%
\usepackage{algorithm}%
\usepackage{algorithmicx}%
\usepackage{algpseudocode}%
\usepackage{listings}%
\usepackage{makecell} 

\usepackage{thmtools}

\theoremstyle{thmstyleone}%
\newtheorem{theorem}{Theorem}
\theoremstyle{thmstyletwo}%

\theoremstyle{thmstylethree}%
\newtheorem{definition}{Definition}%

\begin{document}

\title[Article Title]{GraphSB: Boosting Imbalanced Node Classification on Graphs through Structural Balance}



\author[1]{\fnm{Chaofan} \sur{Zhu}}\email{}
\author[1]{\fnm{Xiaobing} \sur{Rui}}\email{}
\author*[1,2]{\fnm{Zhixiao} \sur{Wang}}\email{zhxwang@cumt.edu.cn}

\affil[1]{\orgdiv{Department of Computer Science and Technology}, \orgname{China University of Mining and Technology}, \orgaddress{\city{Xuzhou}, \postcode{221116}, \country{China}}}

\affil[2]{\orgdiv{Engineering Research Center of Mine Digitization}, \orgname{Ministry of Education of the People’s Republic of China}, \orgaddress{\city{Xuzhou}, \postcode{221116}, \country{China}}}


\abstract{

Imbalanced node classification on graphs is a hot research topic in graph learning. 
Most existing methods typically utilize Graph Neural Networks (GNNs) to learn node representations for effective classification. 
However, GNN-based methods often operate under the assumption that the structural distributions of graphs are balanced, overlooking the inherent structural imbalance in real-world graphs. 
We argue that structural imbalance has a significant impact on the imbalanced node classification, focusing solely on class quantity balance fails to account for structural balance and thereby undermines the effectiveness of quantity-based balancing strategies. 
To confirm such a point, this paper provides a detailed theoretical analysis demonstrating that inherent structural imbalance causes GNNs to encounter imbalanced neighborhood structures and biased message-passing mechanisms during node embedding updates. Consequently, minority-class node embeddings suffer from assimilation, since majority-class nodes dominate feature aggregation, leading to accumulated deficiencies during further node synthesis and ultimately limiting the effectiveness of unbalanced node classification. Based on these theoretical insights, we propose a comprehensive imbalanced node classification framework named GraphSB (Graph Structural Balance) that incorporates structural balance as a key strategy addressing graph structural bias before node synthesis. The proposed strategy contains a two-stage graph structure optimization: structure enhancement and relation diffusion. The structure enhancement component strengthens the connectivity of minority class nodes while minimizing structural damage to them; the relation diffusion component effectively reduces noise and expands receptive fields without overfitting. Thus, minority-class nodes can obtain rich information, resulting in high-quality node embeddings. Extensive experiments on five datasets demonstrate that GraphSB significantly outperforms the state-of-the-art methods, achieving improvements of 3.9\% and 2.5\% in ROC-AUC scores on Cora and BlogCatalog datasets, respectively. More importantly, the proposed Structural Balance can be seamlessly integrated into state-of-the-art methods as a simple plug-and-play module, increasing their performance by an average of 3.67\%.


}


\vspace{12pt} 

\maketitle

\section{Introduction}\label{sec1}

Graph is a non-Euclidean data structure used to describe complex relational networks. In recent years, Graph Neural Networks (GNNs)\cite{1} have become a powerful tool for graph representation learning with their end-to-end approach to passing and receiving messages to efficiently integrate attribute and structural information. GNNs have achieved advanced performance in various graph-related downstream tasks\cite{2,3,4}. Among them, node classification\cite{5} is one of the most fundamental tasks that plays a critical role in numerous real-world applications, such as drug discovery\cite{drugdiscovery} and fraud detection\cite{frauddetection}. The goal of node classification\cite{6} is to predict the labels of the unlabeled nodes based on the characteristics of the nodes.

In fact, real-world graphs often exhibit a significant class imbalance, with the number of samples varying greatly across different classes \cite{7,liu2023survey}. Thus, imbalanced node classification has emerged as a critical challenge in graph learning. To address this challenge, researchers have developed various approaches. 
Most of them attempt to utilize data-level techniques\cite{zhao2021graphsmote,9,13,liu2025cdcgan} to synthesize  minority nodes for balancing the training process. GraphSMOTE \cite{zhao2021graphsmote} utilizes a GNN encoder to learn node embeddings and applies SMOTE\cite{chawla2002smote} to generate synthetic minority nodes. GraphENS\cite{9} prevents overfitting of synthesized minority nodes by incorporating information from majority classes to generate minority samples. GraphMixup\cite{13} constructs a semantic relation space for the synthesis of minority nodes. CDCGAN\cite{liu2025cdcgan} employs conditional GAN-based generation to ensure diversity and distinguishability of the synthetic nodes.

These above methods address imbalanced node classification by focusing on class quantity balance. However, they risk assimilating the embeddings of minority nodes into majority patterns, leading to unrepresentative sample generation and fundamentally limiting their effectiveness\cite{topping2021understanding}. This is largely because GNNs \cite{5,23,25} perform well on balanced graphs but struggle with structural imbalance, often exacerbating the performance degradation of minority-class nodes. Therefore, researchers have attempted to address the imbalanced node classification from the perspective of algorithm-level. For example, Chen et al.\cite{ReNode} propose topology-aware node reweighting based on influence conflict, reducing weights for boundary nodes. Song et al.\cite{TAM} adjust classification margins using local neighbor distributions but ignore global structural dominance, and Fu et al.\cite{HyperIMBA} leverage hyperbolic geometry to model hierarchy imbalance. However, these methods focus on post hoc adjustments (e.g. loss reweighting or margin tuning) without modifying the inherently imbalanced graph structure, which is insufficient to address the fundamental challenges of imbalanced node classification. In fact, the structural imbalanced graphs typically shows strong asymmetry, minority-class nodes often exhibit sparser neighborhoods and weaker community affiliations, inducing biased neighbor aggregation during message passing, and ultimately resulting in a significantly degraded representation quality of minority-class nodes.



Therefore, we argue that structural imbalance has a significant and non-negligible impact on imbalanced node classification, and focusing solely on quantity balance fails to account for structural imbalance in graphs, a fundamental limitation that inherently undermines the effectiveness of quantity balance. As illustrated in Figure~\ref{fig1}, we demonstrate two typical scenarios in imbalanced graph learning. 
In the first case (top row), despite the quantity imbalance in selected nodes, the balanced graph structure allows GNNs to learn discriminative features. 
In contrast, the second case (bottom row) shows how structural imbalance in graph disrupts the learning process: the imbalanced graph structure affects the graph learning process, leading to inherently biased embeddings that exacerbate bias accumulation in subsequent synthesis. This phenomenon fundamentally stems from the intrinsic coupling between node embedding and graph structure, necessitating early-stage intervention rather than post hoc compensation of biased embeddings.

\begin{figure}[ht]
\centering
\includegraphics[width=0.5\textwidth]{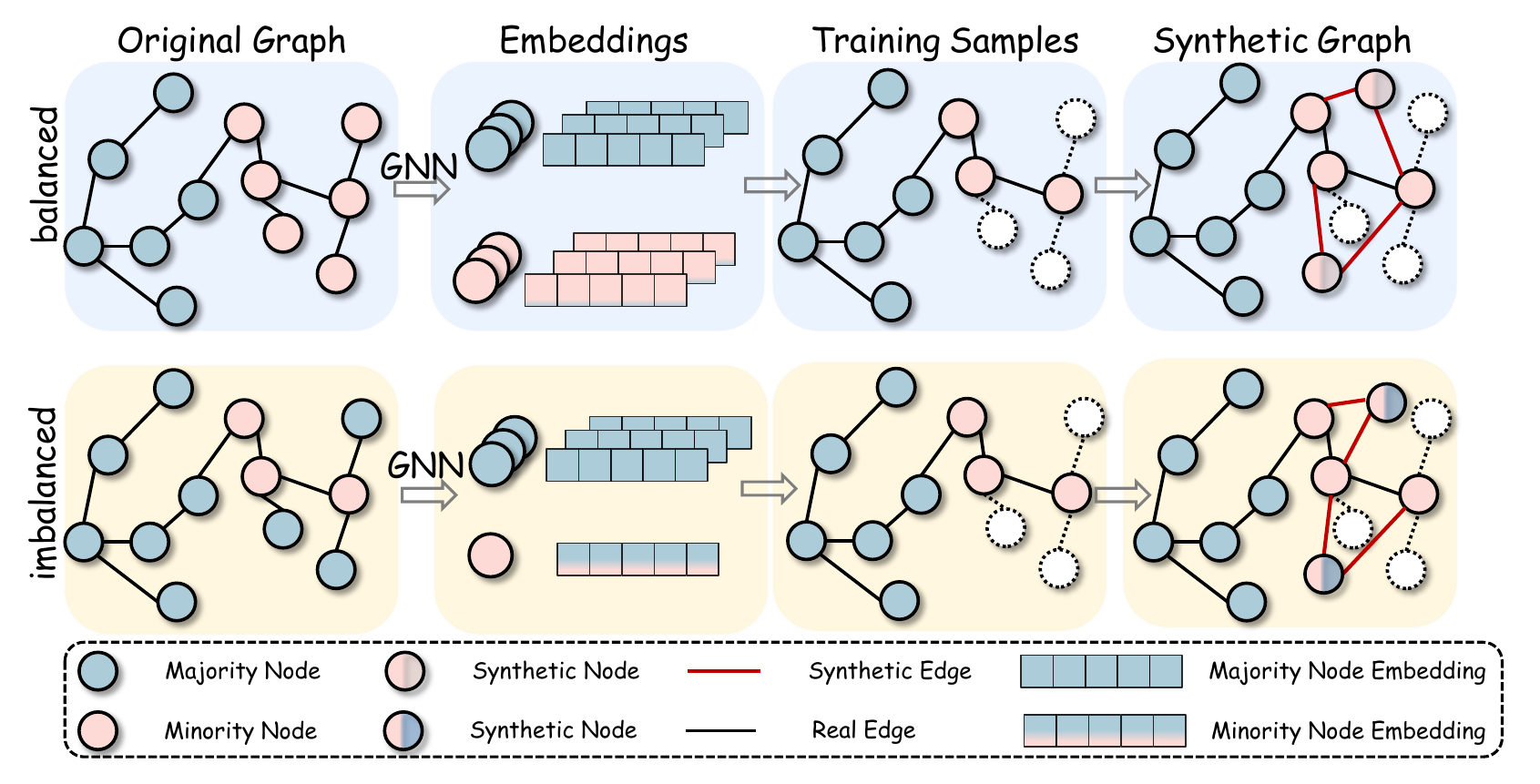}
\caption{Imbalanced node learning on graphs}
\label{fig1}
\end{figure}


Theoretical analysis of structural imbalance bias in GNNs (as shown in Section\ref{sec:theoretical}) confirms our assertions. There's a "minority squeeze" phenomenon in imbalanced graph structure. This phenomenon manifests in two aspects: first, minority-class nodes' embeddings are heavily dominated by high-degree nodes and majority-class nodes' influences; second, during message passing, noise signals from different classes continuously accumulate, leading to node homogenization, while the narrow neighbor aggregation lacks diverse embeddings, thereby further weakening the feature representation capability of minority-class nodes. 

Motivated by these observations, we propose a comprehensive imbalanced node classification framework named \textbf{GraphSB} (Graph Structural Balance) that incorporates Structural Balance as a key strategy to address the structural imbalance of graphs before node synthesis. Specifically, SB contains a two-stage graph structure optimization process: (1) \textbf{Structure Enhancement}, which adaptively strengthens connections between nodes of the minority class by establishing edges with similar nodes in their local neighborhoods. The edge density is dynamically adjusted based on the average degree of majority-class nodes, ensuring balanced structural representation; and (2) \textbf{Relation Diffusion}, which applies multi-step diffusion to expand the receptive field of minority-class nodes. During training, a portion of these diffused edges are randomly dropped in each epoch to enhance robustness against noise and prevent overfitting. The remaining edges are incorporated into the graph neural network for node embedding updates. 

Our contributions are summarized as follows:
\begin{itemize}
    \item \textbf{Preliminary Theoretical Analysis}. We argue that structural imbalance has a significant and non-negligible impact on the imbalanced node classification task. Structural imbalance cannot be remedied by balancing class quantities, and more critically, structural imbalance undermines quantity balance. This insight is confirmed by our theoretical analysis, which reveals that inherent structural imbalance causes GNNs to encounter imbalanced neighborhood structures and biased message-passing mechanisms during node embedding updates. Consequently, minority-class node embeddings suffer from assimilation, since majority-class nodes dominate feature aggregation, leading to accumulated deficiencies during further node synthesis. These learning deficiencies further accumulate during node synthesis, undermining the effectiveness of imbalanced node classification.
    \item \textbf{Model Design}. We propose GraphSB, a comprehensive framework for imbalanced node classification that incorporates structural balance as a key strategy to address graph structural imbalance before node synthesis. GraphSB contains a two-stage graph structure optimization process: the structure enhancement component strengthens minority node connectivity while minimizing structural damage to minority class nodes; the relation diffusion component reduces noise and expands receptive fields without overfitting. This allows minority nodes to obtain rich information, resulting in high-quality node embeddings.
    \item \textbf{Experimental Evaluation}. Extensive experiments on five datasets demonstrate that GraphSB significantly outperforms state-of-the-art methods, achieving improvements of 3.9\% and 2.5\% in ROC-AUC scores on Cora and BlogCatalog datasets, respectively. More importantly, the proposed the proposed Structural Balance can be seamlessly integrated into state-of-the-art methods as a simple plug-and-play module, increasing their performance by an average of 3.67\%.
\end{itemize}

\section{Related Work}
Imbalanced node classification has been extensively studied, especially with the help of GNN models~\cite{johnson2019survey,qu2021imgagn,rout2018handling}. Existing approaches can be broadly categorized into two groups:data-level methods that focus on rebalancing class distributions and algorithm-level methods that optimize the learning process to address class imbalance.

\subsection{Data-level Methods}
In the imbalanced node classification task, some classes of nodes have significantly fewer instances\cite{han2022g,huang2022graph}. Data-level methods have emerged to alleviate the quantity imbalance of different classes\cite{mani2003knn,2019generative,ding2022data}. Traditional data-level solutions such as SMOTE\cite{chawla2002smote} employ resampling techniques to generate synthetic samples. GNNs-based solutions such as GraphSMOTE\cite{zhao2021graphsmote} and ImGAGN\cite{qu2021imgagn} extend these ideas by generating synthetic nodes. GraphMixup\cite{13} extends Mixup\cite{14} to synthesize minor nodes in the embedding space and then determine the neighborhood of newly synthesized nodes with a pre-trained edge predictor. GraphENS\cite{9} synthesizes minority nodes by incorporating information from majority classes. GraphSR\cite{graphsr} provides a more comprehensive approach to handle intertwined imbalances. Although these methods are simple, they focus mainly on quantity rebalancing\cite{xu2024imbalance} while overlooking the inherent structural imbalance of graphs, fundamentally limiting their effectiveness in imbalanced node classification.

\subsection{Algorithm-level Methods}
Algorithm-level methods seek to directly increase the importance of minority classes with appropriate penalty functions\cite{ReNode,11,cao2019learning,wang2022fair} or model regularization\cite{RA-GCN,DR-GCN,qian2022co,nt2019revisiting,yan2023rethinking} to address class imbalance. For example, ReNode\cite{ReNode} incorporates modified loss functions or penalty terms designed to emphasize minority classes. RA-GCN\cite{RA-GCN} proposes learning a parametric penalty function through adversarial training, which reweights samples to help the classifier fit better between the majority and minority classes, thus avoiding bias toward either class. DR-GCN\cite{DR-GCN} tackles the class-imbalanced problem by imposing two types of regularization, combining conditional adversarial training with latent distribution alignment. GNN-INCM\cite{GNN-INCM} attempts to address class imbalance through clustering and graph reconstruction to optimize the embeddings of the minority class node. However, these approaches focus mainly on optimizing the embedding space while overlooking the imbalanced structure in graphs. Recent studies Graph-DAO \cite{xia2024novel} have revealed that minority nodes are inherently more susceptible to biased message passing, and structure imbalance\cite{liu2024class} is further amplified during subsequent node-level rebalancing. Despite these efforts, algorithm-level approaches often struggle to provide a reliable solution, as they typically focus on post hoc adjustments without addressing the fundamental structural imbalance in graphs.

\section{Analysis of Structural imbalance Bias in GNNs}
\label{sec:theoretical}

\subsection{Preliminary}
Let $\mathcal{G}=(V,E,X)$ denote an attributed graph with node set $V$, edge set $E$, and feature matrix $X \in \mathbb{R}^{n \times p_0}$, where $n = |V|$ and $x_i \in \mathbb{R}^{p_0}$ represents the feature vector of node $i \in \{1,\cdots,n\}$. 
In Message Passing Neural Networks (MPNNs)~\cite{MPNNs}, the node representation $h_i^{(\ell)} \in \mathbb{R}^{p_\ell}$ at layer $\ell \geq 0$ evolves through message passing operations. 
Let $\phi_\ell$ and $\psi_\ell$ denote learnable update and aggregation functions, respectively.
The layer-wise propagation rule is then given by:
\begin{equation}
h_i^{(\ell+1)} = \phi_\ell\left(h_i^{(\ell)}, \sum_{j \in \mathcal{N}(i)} \hat{A}_{ij} \psi_\ell(h_i^{(\ell)}, h_j^{(\ell)})\right)
\end{equation}
where $\hat{A} = \tilde{D}^{-1/2}\tilde{A}\tilde{D}^{-1/2}$ is the normalized adjacency matrix with self-loops, $\tilde{A} = A + I$ is the adjacency matrix with self-connections, and $\tilde{D}$ is the degree matrix of $\tilde{A}$. 

Consider a class-imbalanced graph $\mathcal{G}$ with $V = V_1 \cup V_2$ ($V_1\cap V_2=\varnothing$), where $V_1$ and $V_2$ represent the nodes of minority and majority class, respectively, with the imbalance ratio $\beta = \frac{|V_2|}{|V_1|} \gg 1$. 
Under a Stochastic Block Model~\cite{SBM} SBM$(n, p, q)$ with intra-class edge probability $p$ and inter-class probability $q$, where $p \gg q$, the expected degrees are:

\begin{itemize}
    \item For minority class nodes: $\mathbb{E}[d_{v_1}] = n_1 p + n_2 q = n_1 (p + q \beta)$
    \item For majority class nodes: $\mathbb{E}[d_{v_2}] = n_1 q + n_2 p = n_1 (q + p \beta)$
\end{itemize}

We define the degree disparity factor $\gamma = \frac{\mathbb{E}[d_{v_2}]}{\mathbb{E}[d_{v_1}]} = \frac{q + p\beta}{p + q\beta} \geq 1$, which quantifies the structural advantage of majority nodes.

\subsection{Majority Node Dominance}

In this subsection, we analyze how the structural bias in GNNs impacts minority class node representations. We demonstrate that under class imbalance, two key mechanisms significantly disadvantage minority nodes: information dilution through over-squashing and gradient dominance from majority class nodes.

\textbf{The over-squashing phenomenon} in GNNs, first analyzed in \cite{bottleneck}, manifests as exponential decay of information from distant nodes. Consider an $L$-layer MPNN with node features $X \in \mathbb{R}^{n\times p_0}$, message functions $\psi_\ell$, and update functions $\phi_\ell$ satisfying $\|\nabla\phi_\ell\| \leq \alpha$, $\|\nabla\psi_\ell\| \leq \beta$. For nodes $i, s \in V$ with geodesic distance $d_\mathcal{G}(i,s) = r+1$, the Jacobian sensitivity exhibits:

\begin{equation}
\left|\frac{\partial h_i^{(r+1)}}{\partial x_s}\right| \leq (\alpha\beta)^{r+1}(\hat{A}^{r+1})_{is}
\label{eq:jacobian_bound}
\end{equation}

For example, in bottleneck structures like $b$-ary trees, $(\hat{A}^{r+1})_{is} = \mathcal{O}(b^{-r})$ confirms exponential information distortion.

\begin{definition}
\textbf{Information Dilution} refers to the phenomenon where information propagation between nodes decays rapidly as the distance increases. In class-imbalanced graphs, this decay is further accelerated by the degree disparity caused by majority class nodes, with the attenuation factor amplified by $\gamma^{-l}$.
\end{definition}

The relationship between Information Dilution and the over-squashing phenomenon is formally quantified in the following Theorem~\ref{thm:info_dilution}.

\begin{theorem}
\label{thm:info_dilution}
For a path of length $l$ in a class-imbalanced graph with imbalance ratio $\beta$ and degree ratio $\gamma$, the expected path weight satisfies:

\begin{equation}
\mathbb{E}[\mathcal{W}_l] = \frac{W^{(l)}}{\mathbb{E}[d_{v_1}]^l}(\frac{\beta+\gamma}{\gamma(\beta+1)})^l
\end{equation}
where  $\mathcal{W}_l$ denotes the weight of information propagated along a path of length $l$. $W^{(l)} = \prod_{\ell=1}^l \nabla\phi_\ell \nabla\psi_\ell$ is the cumulative weight matrix product.

\end{theorem}

\begin{proof}
The class-conditional expectations for message propagation are characterized as follows: (1) Majority paths ($V_2$): $\mathbb{E}\left[\frac{\|W^{(l)}\|}{N_{\text{maj}}}\right] = \frac{\|W^{(l)}\|}{\gamma\mathbb{E}[d_{v_1}]}$, and (2) Minority paths ($V_1$): $\mathbb{E}\left[\frac{\|W^{(l)}\|}{N_{\text{min}}}\right] = \frac{\|W^{(l)}\|}{\mathbb{E}[d_{v_1}]}$.
The probability of randomly selecting a node from each class is:
\begin{equation}
\begin{aligned}
p_1 &= P(v \in V_1) = \frac{1}{\beta+1}\\
p_2 &= P(v \in V_2) = \frac{\beta}{\beta+1}
\end{aligned}
\end{equation}

For a path of length $l$, let $k_p$ represent the number of steps passing through majority class nodes. Applying the binomial theorem, the path weight expectation is:
\begin{equation}
\begin{aligned}
\mathcal{W}_l &= \frac{W^{(l)}}{\mathbb{E}[d]^l} \sum_{k_p=0}^{l} \binom{l}{k_p} \left(\frac{p_2}{\gamma}\right)^{k_p} p_1^{l-k_p}\\
&= \frac{W^{(l)}}{\mathbb{E}[d_{v_1}]^l} \sum_{k_p=0}^{l} \binom{l}{k_p} \left(\frac{\beta}{\gamma(\beta+1)}\right)^{k_p} \left(\frac{1}{\beta+1}\right)^{l-k_p}\\
&= \frac{W^{(l)}}{\mathbb{E}[d_{v_1}]^l} \left(\frac{\beta + \gamma}{\gamma(\beta+1)}\right)^{l}
\end{aligned}
\end{equation}

Thus concludes the proof.

\end{proof}

Theorem~\ref{thm:info_dilution} demonstrates that majority class nodes, with their higher degree ($\gamma$ times higher than minority nodes), accelerate information dilution, causing minority node features to be overwhelmed in message passing.

To better characterize the mechanism by which the majority class nodes dominate the learning process, we formally introduce the concept of Gradient Dominance as follows.

\begin{definition}
\textbf{Gradient Dominance} refers to the phenomenon in class-imbalanced graphs where the gradients associated with majority class nodes overwhelmingly influence the parameter updates during training, thereby suppressing the learning signals from minority class nodes.
\end{definition}

The effect of Gradient Dominance is formally quantified in the following Theorem~\ref{thm2}.

\begin{theorem}
\label{thm2}
Assume path weight $\mathcal{W}_l$ is independent of node degrees, then 
the expected ratio of gradient magnitudes between a minority class node $u \in V_1$ and a majority class node $v \in V_2$ at depth $L$ is:
\begin{align}
\frac{\mathbb{E}[\|\nabla_\theta \mathcal{L}(u)\|]}{\mathbb{E}[\|\nabla_\theta \mathcal{L}(v)\|]} &\propto 
\frac{1}{\beta}
\end{align}
\end{theorem}

\begin{proof}
The expected number of nodes at the $l$-th level is:
\begin{equation}
\begin{aligned}
\mathbb{E}[N^{(l)}] &= \mathbb{E}[d_{v_1}]^{l}\sum_{k_p=0}^{l} \binom{l}{k_p} \left({p_2}{\gamma}\right)^{k_p} p_1^{l-k_p}\\
&= \mathbb{E}[d_{v_1}]^l \cdot \left(\frac{1 + \beta\gamma}{\beta+1}\right)^l
\end{aligned}
\end{equation}

Class contributions can be expressed as: Minority: $\mathcal{W}_l \times \frac{1}{\beta+1} \times \mathbb{E}[N^{(l)}]$ and Majority: $\mathcal{W}_l \times \frac{\beta}{\beta+1} \times \mathbb{E}[N^{(l)}]$. Taking the ratio, we obtain:

\begin{align}
\frac{\mathbb{E}[\|\nabla_\theta \mathcal{L}(u)\|]}{\mathbb{E}[\|\nabla_\theta \mathcal{L}(v)\|]} &\propto \frac{\text{Minority }}{\text{Majority }} 
= \frac{1}{\beta}
\end{align}

Thus concludes the proof.
\end{proof}

Theorem~\ref{thm2} demonstrates that the numerical advantage and higher degree of majority class nodes make them dominant in information propagation and gradient updates.

\subsection{Minority Feature Assimilation}

In this section, we analyze how minority node features become progressively assimilated during neighborhood aggregation.

\begin{definition}
\textbf{Feature Assimilation} is the progressive convergence of minority class node representations toward majority class centroids during GNN message passing, the distinctive characteristics of minority nodes are overwhelmed by the predominant influence of majority class nodes, leading to reduced class separability in deeper layers.
\end{definition}

\begin{theorem}
\label{thm:feature_conv}
In class-imbalanced graphs with imbalance ratio $\beta \gg 1$, the centroid distance $\Delta^{(l)} = \mu_1^{(l)} - \mu_2^{(l)}$ between minority and majority classes decays exponentially with network depth:
\begin{equation}
\|\Delta^{(l)}\| \leq C \cdot (\sigma_{\max}(W)\lambda_2(M))^l\|\Delta^{(0)}\|
\end{equation}
where $C$ is a constant, $\Delta^{(0)}$ is the initial centroid distance, $\sigma_{\max}(W)$ represents the maximum singular value of weight matrix $W$, and $\lambda_2(M)$ denotes the second eigenvalue of the propagation matrix $M$ with $\sigma_{\max}(W)\lambda_2(M) < 1$.
\end{theorem}
For example, setting $p = 0.5$, $q = 0.1$, and $\beta = 10$, we obtain
\[
M = \begin{bmatrix}
0.333 & 0.196 \\
0.067 & 0.980
\end{bmatrix},
\quad \lambda_1 \approx 1, \quad \lambda_2 \approx 0.313.
\]
Since $\lambda_2 < 1$, it governs the convergence rate. By spectral decomposition, the magnitude of $\Delta^{(l)}$ decays as $\mathcal{O}(0.313^l)$, demonstrating how minority features rapidly assimilate toward majority class representations with increasing network depth.

\begin{proof}

In a GNN, the representation of node $u$ at layer $l$ is updated as:
\begin{equation}
h_u^{(l)} = W^{(l)} \sum_{v \in \mathcal{N}(u)} \frac{h_v^{(l-1)}}{\sqrt{d_u d_v}}
\end{equation}

where $\mathcal{N}(u)$ denotes the neighbors of $u$, $d_v$ is the degree of node $v$, and $W^{(l)}$ is the layer-specific weight matrix. minority class centroids are defined as:
\begin{equation}
\mathbb{E}[h_u^{(l)}] = W^{(l)} \left( \frac{p}{p+q\beta} \mu_1 + \frac{q\beta}{\sqrt{(p+q\beta)(q+p\beta)}} \mu_2 \right) 
\end{equation}

when overestimating minority influence (via $q+p\beta>\sqrt{(p+q\beta)(q+p\beta)}$),
For a node $u \in V_1$, the expected update is:
\begin{equation}
\mathbb{E}[h_u^{(l)}] \approx W^{(l)} \left( \frac{p}{p + q \beta} \mu_1^{(l-1)} + \frac{q \beta}{q + p \beta} \mu_2^{(l-1)} \right)
\end{equation}
Similarly, for $u \in V_2$:
\begin{equation}
\mathbb{E}[h_u^{(l)}] \approx W^{(l)} \left( \frac{q}{p + q \beta} \mu_1^{(l-1)} + \frac{p \beta}{q + p \beta} \mu_2^{(l-1)} \right)
\end{equation}

Define the state vector $\mathbf{z}^{(l)} = [\mu_1^{(l)}, \mu_2^{(l)}]^\top$. The GNN's message passing induces:
\begin{equation}
\mathbf{z}^{(l)} = M W \mathbf{z}^{(l-1)}, \quad 
M = \begin{bmatrix}
\frac{p}{p + q\beta} & \frac{q\beta}{q + p\beta} \\
\frac{q}{p + q\beta} & \frac{p\beta}{q + p\beta}
\end{bmatrix}
\end{equation}

After $l$ layers: 
\begin{equation}
\mathbf{z}^{(l)} = (WM)^l \mathbf{z}^{(0)}
\end{equation}

The convergence rate is governed by $\rho(WM) \leq \sigma_{\max}(W)\lambda_2(M)$, where $\lambda_2(M)$ is the subdominant eigenvalue of $M$.The centroid distance $\Delta^{(l)} = \mu_1^{(l)} - \mu_2^{(l)}$ decays as:
\begin{equation}
\|\Delta^{(l)}\| \leq C \cdot \underbrace{(\sigma_{\max}(W)\lambda_2(M))^l}_{\text{exponential decay}} \|\Delta^{(0)}\|
\end{equation}
The condition $\sigma_{\max}(W)\lambda_2(M) < 1$ ensures progressive assimilation of minority features.\\
Thus concludes the proof.
\end{proof}

\subsection{Conclusion}
\begin{itemize}
    \item \textbf{Majority Dominance Effect:} High-degree majority nodes with numerical superiority dominate message propagation paths, while minority nodes with sparser connectivity have restricted receptive fields, leading to weaker information aggregation and reduced representational power in the embedding space

    \item \textbf{Minority Assimilation Effect:} Minority node features become progressively assimilated during neighborhood aggregation, as their distinctive characteristics are overwhelmed by the predominant influence of majority-class nodes through iterative message passing.
\end{itemize}

\section{Methodology}
\label{sec4}

Based on the research findings presented in Section 3, we propose the GraphSB framework to mitigate the cross-impact of minority node embeddings caused by both quantitative imbalance and structural imbalance inherent in graph topology. GraphSB is designed to alleviate GNN-induced errors and biases, aiming to enhance the accurate representation of node embeddings and class distributions while preserving the intrinsic topological and attribute information of the data.
fig 3 depicts the outcomes of ablation experiments conducted on three
datasets. From the outcomes, we have the following findings:

An overview of the proposed GraphSB framework is depicted in Fig ~\ref{framework}. Our approach prioritizes structural balance as a key strategy and comprises two core components: (1) similarity-based Graph Structure Enhancement that adaptively strengthens connections between minority nodes, and (2) Relation Diffusion that expands receptive fields while preventing noise accumulation. Together, these components effectively mitigate the embedding quality degradation of minority nodes caused by structural bias in GNNs. For quantity balancing strategies, we follow the same approach as GraphMixup\cite{13}.

\begin{figure*}[!ht]
\centering
\includegraphics[width=\textwidth]{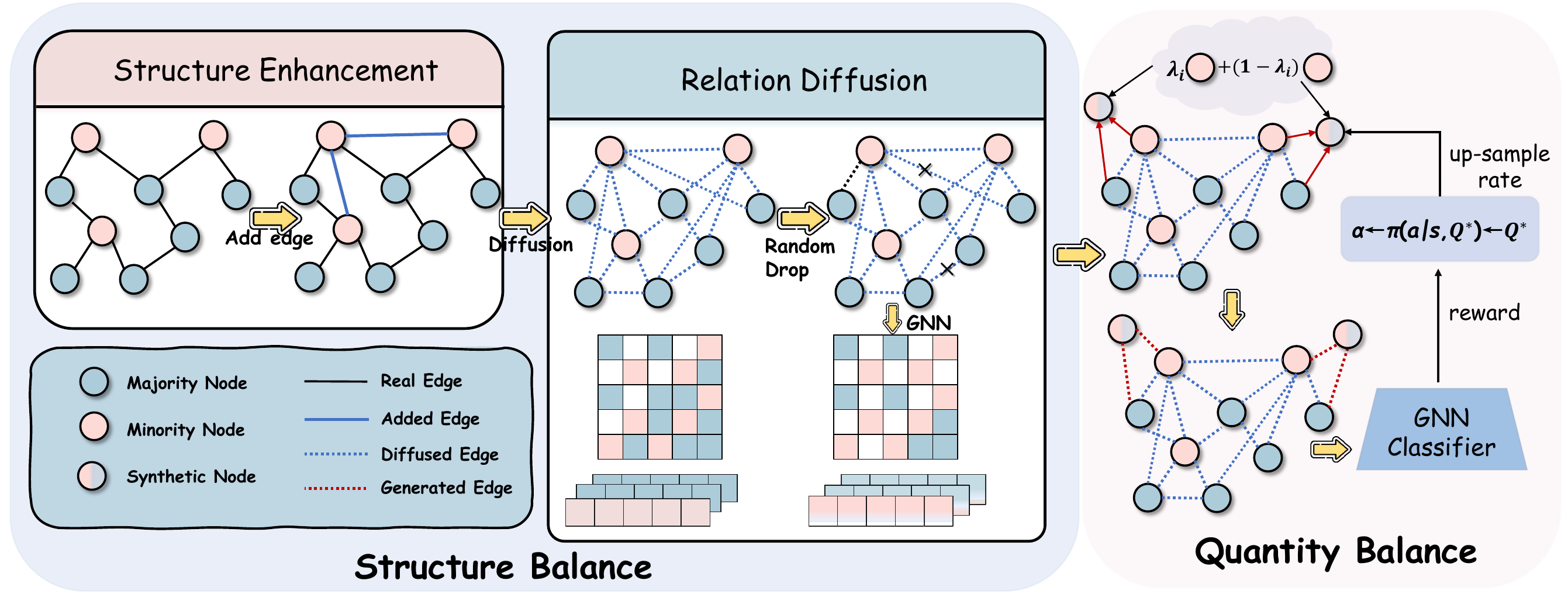}
\caption{Overview of the proposed GraphSB framework, which incorporates two key steps to enhance the representation of imbalanced data in graph learning.}
\label{framework}
\end{figure*}

tr\subsection{Graph Structure Enhancement}  
Structural imbalance in graphs arises from the inherent sparsity of intra-class edges among minority nodes, which hinders effective message passing and leads to suboptimal node embeddings. To address this structural imbalance in class-imbalanced graph learning, we propose a novel framework that reconstructs graph structure through adaptive local enhancement of minority class node connectivity, thereby facilitating effective propagation of intra-class information. This structural enhancement mechanism improves the embedding quality of minority nodes and obtains more representative samples to enhance the model's learning performance.

First, we identify minority classes by selecting the $k$ classes with the least samples in the training set:
\begin{equation}
\mathcal{C}_m = \{c \mid c \in \text{argsort}(|\{v \in \mathcal{T} \mid y_v = c\}|)[0:k]\}
\end{equation}
For minority nodes that require edge enhancement, we employ a localized adaptive strategy that preserves structural integrity while improving connectivity. To achieve accurate node selection for edge augmentation, we first calculate the cosine similarity between nodes: $\text{sim}(i,j) = \frac{\mathbf{x}_i^T \mathbf{x}_j}{||\mathbf{x}_i|| \cdot ||\mathbf{x}_j||}$ where $\mathbf{x}_i$ and $\mathbf{x}_j$ are the L2-normalized feature vectors.

However, enforcing arbitrary structural enhancements across distant regions of the graph can be structurally unfair and inevitably leads to information loss and degradation of the original graph topology. To preserve the inherent structural properties and prevent the corruption of long-range dependencies, we constrain our edge enhancement to local neighborhoods. We further incorporate locality constraints by computing the distance-aware similarity $\text{sim}_{\text{local}}(i,j) = \text{sim}(i,j) \cdot \mathbb{I}[d_{\mathcal{G}}(i,j) \leq 4]$, where $\mathbb{I}[\cdot]$ is the indicator function and $d_{\mathcal{G}}(i,j)$ is the shortest path distance. Specifically, we define the candidate set within the 4-hop neighborhood of each minority node:
\begin{equation}
\mathcal{N}_v^{\text{local}} = \{j \mid j \in \mathcal{V} \setminus \{v\}, A_{v,j}=0, d_{\mathcal{G}}(v,j) \leq 4\}
\end{equation}
Within this local constraint, for each minority node, our adaptive edge enhancement strategy calculates the degree difference between majority and minority nodes and selects the most similar nodes from the local candidate set:
\begin{equation}
\mathcal{N}_v = \text{top-}\Delta d_v(\mathcal{N}_v^{\text{local}}, \text{sim}(\cdot, v))
\end{equation}
where $\Delta d_v = \max(0,\bar{d}_{\text{maj}} - d_v)$ represents the adaptive number of edges to add, $\bar{d}_{\text{maj}}$ is the average degree of majority nodes, and $d_v$ is the current degree of node $v$.

The updated adjacency matrix is defined as:
\begin{equation}
A'_{v,u}  =
\begin{cases}
1, & \text{if } u \in \mathcal{N}_v \text{ or } v \in \mathcal{N}_u \\
A_{v,u}, & \text{otherwise}
\end{cases}
\end{equation}

This localized approach ensures that minority nodes achieve connectivity comparable to majority nodes while maintaining semantic relevance through similarity-guided edge selection and preserving the fundamental structural integrity of the original graph topology.

\subsection{Relation Diffusion Process}
After adding intra-class connections for minority nodes, we refine the connectivity structure using a multi-step diffusion process. In a standard adjacency matrix, a direct connection is represented by a value of 1, while no connection is denoted by 0. However, such a binary representation is often too rigid to capture the nuanced relationships in the graph. To overcome this, we iteratively update the connectivity using a diffusion operator:
\begin{itemize}
    \item At each diffusion step, the weights of existing edges are gradually decayed.
    \item Indirect relationships, which were originally absent (i.e., marked by 0), begin to accumulate non-zero values through propagation from neighbors and their neighbors.
\end{itemize}

The symmetrically normalized adjacency matrix $\mathbf{\hat{A}} = \mathbf{D}^{-\frac{1}{2}}(\mathbf{A} + \mathbf{I})\mathbf{D}^{-\frac{1}{2}}$ is computed using the diagonal degree matrix $\mathbf{D}$ where $D_{ii} = \sum_j (A_{ij} + I_{ij})$, which prevents numerical instability during diffusion. The generalized diffusion process forms a geometric series $\mathbf{S} = \alpha \sum_{k=0}^{\infty} (1-\alpha)^k \mathbf{\hat{A}}^k$ with teleport probability $\alpha \in (0,1)$, converging to $\mathbf{S} = \alpha(\mathbf{I} - (1-\alpha)\mathbf{\hat{A}})^{-1}$.

For computational efficiency, we approximate this through iterative updates:
\begin{equation}
    \mathbf{S}^{(0)} = \mathbf{I}, \quad \mathbf{S}^{(t+1)} = \alpha\mathbf{\hat{A}}\mathbf{S}^{(t)} + (1-\alpha)\mathbf{S}^{(t)}
\end{equation}

To enhance the robustness and generalization capability of node embeddings, we incorporate a symmetric edge dropout strategy during the training phase. For the diffusion matrix $\mathbf{S}$, we construct a binary mask $\mathbf{M} \in \{0,1\}^{n \times n}$ by sampling the upper triangular elements from a Bernoulli distribution with success probability $1-p_{\text{drop}}$, where $p_{\text{drop}}=0.1$:

\begin{equation}
\mathbf{M}_{ij}^{\text{upper}} = 
\begin{cases}
1, & \text{if } r_{ij} > p_{\text{drop}} \\
0, & \text{otherwise}
\end{cases}
\end{equation}

The masked diffusion matrix $\mathbf{\tilde{S}}$ is then computed through element-wise application of the symmetric mask with expectation-preserving scaling:

\begin{equation}
\mathbf{\tilde{S}} = \frac{\mathbf{S} \odot \mathbf{M}}{1-p_{\text{drop}}}
\end{equation}

where $\odot$ denotes the Hadamard product. This stochastic regularization prevents overfitting to specific structural patterns while promoting diverse node embeddings.

Once the enhanced graph structure is obtained, the node embeddings $\mathbf{h}$ are updated:
\begin{equation}
\mathbf{h} = \text{dropout}\left(\sigma\left(\mathbf{\tilde{S}}\mathbf{X}\mathbf{W} + \mathbf{b}\right), p\right)
\end{equation}
where $\mathbf{X}$ is the input feature matrix, $\mathbf{W}$ is the learnable weight parameters, $\mathbf{b}$ is the bias term, with feature dropout rate $p$, and $\sigma$ refers to the nonlinear activation function.

\subsection{Quantity Balancing Strategy}
\label{sec:quantity_balance}

For quantity balancing, it adopts the same approach as GraphMixup\cite{13}, which consists of two main components: minority node generation and edge prediction for topology reconstruction. Once embeddings of node $v$ are learned through structural balance, the target minority node $v$ and its nearest neighbor $nn(v)$ is performed to generate new minority nodes $v' \in V_S$:

\begin{equation}
\begin{aligned}
\mathbf{h}'_i = \lambda\mathbf{h}_i + (1-\lambda)\mathbf{h}_{\text{nn}(i)}\\\quad \text{nn}(i) = \underset{j \in \{\mathcal{V}\setminus i\}, y_j=y_i}{\text{argmin}} \|\mathbf{h}_j - \mathbf{h}_i\|
\end{aligned}
\end{equation}

where $\lambda \in [0,1]$ is a random variable, following uniform distribution. 

The edge predictor determines the connectivity of synthetic nodes using a multi-task learning framework. The framework includes adjacency reconstruction that establishes connections between generated and existing nodes:
\begin{equation}
    \mathcal{L}_{\text{rec}} = \|\widehat{\mathbf{A}} - \mathbf{A}\|_F^2
\end{equation}
where \(\widehat{\mathbf{A}}_{v,u}\) is the predicted connectivity between node \(v\) and \(u\), calculated as:
\begin{equation}
    \widehat{\mathbf{A}}_{v,u} = \sigma(\mathbf{h}_v^T W \mathbf{h}_u)
\end{equation}
where \(\sigma\) is the sigmoid function, \(\mathbf{h}_v\) and \(\mathbf{h}_u\) are the feature vectors of nodes \(v\) and \(u\), and \(W \in \mathbb{R}^F\) is a learnable weight matrix.

Following GraphMixup\cite{13}, the framework also utilizes local structure preservation and global consistency regularization. Local structure preservation encodes neighborhood proximity through path length classification:
\begin{equation}
    \mathcal{L}_{\text{local}} = \frac{1}{|\mathcal{S}|}\sum_{(v,u)\in\mathcal{S}} \text{CE}\left(g_{\theta}([\mathbf{z}_v \odot \mathbf{z}_u]), \tilde{C}_{v,u}\right)
\end{equation}
where $\odot$ is element-wise multiplication, $g_{\theta}$ is a 2-layer perceptron, and $\tilde{C}_{v,u} = \min(d(v,u), 3)$ discretizes path lengths. Global consistency regularization maintains overall topology through cluster-aware distance regression:
\begin{equation}
    \mathcal{L}_{\text{global}} = \frac{1}{|\mathcal{V}|}\sum_{v_i\in\mathcal{V}} \|h_{\phi}(\mathbf{z}_i) - \mathbf{d}_i\|_2^2
\end{equation}
where $\mathbf{d}_i \in \mathbb{R}^T$ contains shortest-path distances to cluster centroids.

The final training objective integrates these components:
\begin{equation}
    \mathcal{L}_{\text{edge}} = \mathcal{L}_{\text{rec}}+ \mathcal{L}_{\text{local}} +\mathcal{L}_{\text{global}}
\end{equation}

Additionally, we incorporate an adaptive rebalancing mechanism via reinforcement learning to dynamically adjust over-sampling scales for different datasets. We formulate the parameter tuning process as a Markov Decision Process (MDP) and employ Q-learning to solve it. The state space represents the current parameter value, with actions that increment or decrement this value. We define a reward function based on accuracy improvement, and use Q-learning with $\epsilon$-greedy exploration:
\begin{equation}
Q(s,a) \leftarrow Q(s,a) + \alpha[r + \gamma \max_{a'} Q(s',a') - Q(s,a)]
\end{equation}
where $\alpha$ is the learning rate and $\gamma$ is the discount factor. This mechanism enables our model to dynamically refine its rebalancing strategy during training, optimizing performance across varying imbalance scenarios and different datasets.

\subsection{Optimization Objective and Training Strategy}
\label{sec3.5}

After obtaining high-quality node embeddings through the Structural Balance process, we employ a GNN-based classifier for node classification. 
\begin{equation}
\mathbf{h}_v^{(L+1)} = \sigma\left(\mathbf{\widetilde{W}}^{(1)} \cdot [\mathbf{h}_v^{(L)} \parallel \mathbf{H}_O^{(L)} \cdot \mathbf{A}_O[:, v]]\right)
\end{equation}

\begin{equation}
\hat{\mathbf{y}}_v = \text{softmax}\left(\mathbf{\widetilde{W}}^{(2)} \cdot \mathbf{h}_v^{(L+1)}\right)
\end{equation}

where $\parallel$ denotes the concatenation operation, $\mathbf{\widetilde{W}}^{(1)} \in \mathbb{R}^{F \times 2F}$ and $\mathbf{\widetilde{W}}^{(2)} \in \mathbb{R}^{M \times F}$ are learnable weight matrices, and $\mathbf{A}_O[:, v]$ represents the connectivity pattern of node $v$. This architecture enhances node classification by incorporating both node features and structural information.

The classifier is optimized using cross-entropy loss on the combined set of original and synthetic labeled nodes:

\begin{equation}
\mathcal{L}_{\text{node}} = \frac{1}{|\mathcal{V}_O|} \sum_{v \in \mathcal{V}_O} \mathcal{L}_{CE}(\hat{\mathbf{y}}_v, y_v)
\end{equation}

where $\mathcal{V}_O = \mathcal{V}_L \cup \mathcal{V}_S$ represents the entire labeled node set, and $\mathcal{L}_{CE}$ is the cross-entropy loss function. During testing, the predicted class for node $v$ is determined as the class with the highest probability.

The overall training objective of GraphSB integrates the structural balance, node synthesis and classification components.  the final objective function can be written as:

\begin{equation}
\min_{\theta, \gamma, \phi} \mathcal{L}_{\text{node}} +  \mathcal{L}_{\text{edge}}
\end{equation}

where $\theta$, $\gamma$, and $\phi$ are the parameters for the feature extractor, edge predictor, and node classifier, respectively.

The training stability of our framework hinges on high-quality embeddings and synthesized edges, necessitating a phased optimization strategy. Initially, the feature encoder and edge generation module undergo structural initialization through exclusive optimization of the edge reconstruction loss $\mathcal{L}_{\text{edge}}$. This pretrained configuration subsequently serves as the foundation for holistic network refinement, where the composite objective $\mathcal{L}_{\text{node}} + \lambda\mathcal{L}_{\text{edge}}$ drives joint parameter updates.

\section{Experiments}
\label{sec4}

In this section, we conduct comprehensive experiments to validate the effectiveness of GraphSB through answering the following research questions:

\begin{itemize}
    \item \textbf{RQ1}: Does GraphSB outperform state-of-the-art imbalanced node classification methods across diverse graph datasets?
    \item \textbf{RQ2}: How robust is GraphSB under varying degrees of class imbalance compared to baseline methods?
    \item \textbf{RQ3}: How do different components (Structure Enhancement and Relation Diffusion) and hyperparameters of the structural balance mechanism affect model performance?
    \item \textbf{RQ4}: Does the  structural balance module generalize to different graph synthesis architectures to achieve better performance?
\end{itemize}

\subsection{Experimental Setup}
\label{sec5.1}

\subsubsection{Datasets}
\label{sec4.1.1}
We evaluate GraphSB on two categories of graph datasets: (1) Citation networks with manual imbalance settings (Cora, Citeseer, PubMed)\cite{sen2008collective} and (2) Social networks with natural imbalance (BlogCatalog, Wiki-CS)\cite{huang2017label,mernyei2020wiki}. Table~\ref{tab:dataset_stats} summarizes the key statistics and imbalance characteristics.

\begin{table}[ht]
\centering
\setlength{\tabcolsep}{4pt}
\caption{Dataset statistics and imbalance characteristics. $\rho$ denotes the imbalance ratio.}
\label{tab:dataset_stats}
\begin{tabular}{lrrrrr}
\toprule
Dataset & Nodes & Edges & Features & Classes & $\rho$ \\
\midrule
Cora & 2,708 & 5,429 & 1,433 & 7 & 0.22 \\
Citeseer & 3,327 & 4,732 & 3,703 & 6 & 0.36 \\
PubMed & 19,717 & 44,338 & 500 & 3 & 0.30 \\
BlogCatalog & 10,312 & 333,983 & 64 & 36 & 0.004 \\
Wiki-CS & 11,701 & 216,123 & 300 & 10 & 0.35 \\
\bottomrule
\end{tabular}
\end{table}

\textbf{Citation Networks.} For Cora and Citeseer,both exhibit a slight class imbalance, with class imbalance ratios of 0.22 and 0.36, respectively. we follow the semi-supervised protocol with controlled imbalance\cite{zhao2021graphsmote}:
\begin{itemize}
    \item select 3 classes as minority ($\mathcal{C}_{\text{min}}$) and the remainder as majority ($\mathcal{C}_{\text{maj}}$)
    \item For $\mathcal{C}_{\text{maj}}$: 20 labeled nodes per class
    \item For $\mathcal{C}_{\text{min}}$: $20 \times \rho$ labeled nodes ($\rho \in [0.1, 0.9]$) with a class imbalance ratio of $\rho$ defaulting to 0.5
\end{itemize}
For PubMed, we select one class as the minority class and the remaining two as majority classes, following the same protocol.
Cora, Citeseer, and PubMed are commonly used as benchmark datasets for imbalanced node classification tasks.

\textbf{Social Networks.} For naturally imbalanced datasets: The BlogCatalog dataset\cite{huang2017label} exhibits genuine class imbalance, with 14 categories having less than 100 instances and 8 categories having more than 500 instances.Wiki-CS\cite{mernyei2020wiki},The class distribution in this dataset is genuinely imbalanced,where we treat classes with fewer samples than the average sample size across all classes as minority classes.
\begin{itemize}
    \item \textbf{BlogCatalog}: 36 classes with extreme imbalance ($\rho=0.004$). Minority classes defined as $\mathcal{C}_{\text{min}} = \{c_i |N(c_i) < 50\}$ containing only 0.4\% of total nodes
    \item \textbf{Wiki-CS}: 10 classes with $\mathcal{C}_{\text{min}} = \{c_i |N(c_i) < \frac{1}{|\mathcal{C}|}\sum_j N(c_j)\}$
\end{itemize}
we split the train/validation/test sets in proportion to 1:1:2 following the study\cite{1}.

\subsubsection{Hyperparameters}
\label{sec4.1.2}

We consider classes with fewer than the average samples per class as minority classes. To achieve different class imbalance ratios, we have varied $\rho \in \{0.1,0.3,0.5, 0.7,0.9\}$. The following hyperparameters are set for all datasets: Adam optimizer with learning rate $lr = 0.001$ and weight decay $decay = 5 \times 10^{-4}$; Maximum Epoch $E = 4000$; Layer number $L = 2$ with hidden dimension $F = 32$; Relation $K = 4$; Loss weights $\beta = 1.0$; Threshold $\eta = 0.5$. In the reinforcement mixup module, we set $\gamma = 1$, $\epsilon = 0.9$, $\Delta \kappa = 0.05$. Besides, the initialization value $\alpha_{\text{init}}^i = \frac{N}{M|C_i|}$ is set class-wise for minority class $C_i \in C_S$ on each dataset. Each set of experiments is run five times with different random seeds, and the average is reported as the metric.

\subsubsection{Baselines}
\label{sec4.1.3}

To evaluate the effectiveness of our proposed structural balancing approach in class-imbalanced scenarios, we select some representative methods as baselines: \textbf{(1) Vanilla\cite{hamilton2017inductive}}: baseline without imbalance handling; \textbf{(2) Over-Sampling} and \textbf{(3) Re-weight\cite{yuan2012sampling+}}: classical rebalancing techniques; \textbf{(4) SMOTE\cite{chawla2002smote}}: traditional synthetic oversampling; \textbf{(5) DR-GNN\cite{DR-GCN}} and \textbf{(6) RA-GNN\cite{RA-GCN}}: adversarial training approaches for imbalanced graphs; \textbf{(7) Embed-SMOTE\cite{ando2017deep}}: embedding-based oversampling; \textbf{(8) GraphSMOTE\cite{zhao2021graphsmote}}: graph-specific minority node synthesis; \textbf{(9) GraphENS\cite{9}}: ego network-based node generation; \textbf{(10) GraphMixup\cite{13}}: semantic-level minority synthesis; \textbf{(11) CDCGAN\cite{liu2025cdcgan}}: generative adversarial approach for graph data. Among these methods, (7)-(11) represent typical graph synthesis approaches. Additionally, we apply our structure balance module as a plug-in to these typical synthesis methods, demonstrating significant improvements across all frameworks.

\subsection{Performance Comparison}
\label{sec5.2}

To comprehensively evaluate the effectiveness of our proposed GraphSB in addressing imbalanced node classification, we conduct extensive experiments on both citation networks (Cora, Citeseer, PubMed)\cite{sen2008collective} and social networks (BlogCatalog, Wiki-CS)\cite{huang2017label,mernyei2020wiki}. We compare GraphSB with various state-of-the-art baselines under different experimental settings. The results demonstrate the superior performance of our relation-balanced approach.

Tables~\ref{tab:table1} and~\ref{tab:table2} present the comparative results on citation networks and social networks respectively. For citation networks, we follow the semi-supervised protocol with controlled imbalance ratio $\rho=0.5$, while for social networks, we utilize their natural imbalance distributions. GraphSB consistently outperforms all baseline methods across different datasets and metrics. For instance, on Cora, GraphSB achieves 81.4\% accuracy, surpassing the second-best method GraphMixup by 3.9\%. The performance gain of GraphSB is also pronounced on naturally imbalanced social networks. On BlogCatalog, which has an extreme imbalance ratio of 0.004, GraphSB improves the F1-score by 12.98\% compared to traditional methods, demonstrating its robustness in handling severe imbalance scenarios. Our relation-balanced approach shows particular advantages in maintaining both local and global structural information. This is evidenced by the consistent improvements in AUC scores across all datasets, indicating better discrimination between classes.

The success stems from our three-stage topology optimization: (1) similarity-based edge construction for minority nodes generates class-aware connections, (2) multi-step diffusion captures higher-order dependencies, and (3) balanced adjacency integration preserves structural integrity while amplifying minority signals. This systematic approach addresses both local neighborhood sparsity and global structural distortion in imbalanced graphs.

\newcommand{\best}[1]{\textcolor{red}{#1}}
\newcommand{\secondbest}[1]{\textcolor{blue}{#1}}

\begin{table*}[!ht]
\centering
\footnotesize
\setlength{\tabcolsep}{1pt}
\caption{Performance comparison on citation networks (Cora, Citeseer and PubMed) using semi-supervised protocol with imbalance ratio $\rho=0.5$.}
\label{tab:table1}
\begin{tabular*}{\textwidth}{@{\extracolsep{\fill}}lccccccccc@{}}
\toprule
\multirow{2}{*}{Methods} & \multicolumn{3}{c}{Cora} & \multicolumn{3}{c}{Citeseer} & \multicolumn{3}{c}{PubMeb} \\
\cmidrule(lr){2-4} \cmidrule(lr){5-7} \cmidrule(lr){8-10}
                            & Acc & AUC & F1 & Acc & AUC & F1 & Acc & AUC & F1 \\
\midrule
Vanilla            & 68.12$\pm$0.43 &91.23$\pm$0.52 &68.14$\pm$0.45 &54.62$\pm$0.41 &81.73$\pm$0.32 &53.75$\pm$0.54 &67.51$\pm$0.34 &85.42$\pm$0.23 &67.43$\pm$0.42 \\
Over-Sampling     &73.23$\pm$0.32 &92.72$\pm$0.31 &72.91$\pm$0.35 &53.54$\pm$0.72 &82.93$\pm$0.34 &53.32$\pm$0.63 &69.21$\pm$0.12 &87.71$\pm$0.04 &69.32$\pm$0.15 \\
Re-weight         &72.92$\pm$0.24 &92.63$\pm$0.32 &72.43$\pm$0.35 &55.52$\pm$0.53 &82.84$\pm$0.32 &55.73$\pm$0.52 &69.72$\pm$0.13 &87.04$\pm$0.05 &69.42$\pm$0.23 \\
SMOTE             &73.24$\pm$0.32 &92.53$\pm$0.34 &73.02$\pm$0.32 &54.03$\pm$0.62 &82.92$\pm$0.33 &53.74$\pm$0.63 &70.82$\pm$0.14 &87.73$\pm$0.12 &70.43$\pm$0.23 \\
Embed-SMOTE       &72.32$\pm$0.23 &91.82$\pm$0.24 &72.23$\pm$0.25 &56.73$\pm$0.64 &83.04$\pm$0.42 &56.82$\pm$0.63 &69.23$\pm$0.32 &86.32$\pm$0.24 &69.13$\pm$0.32 \\
DR-GNN            &74.82$\pm$0.23 &93.32$\pm$0.24 &74.52$\pm$0.25 &55.23$\pm$0.52 &83.23$\pm$0.34 &55.83$\pm$0.62 &68.52$\pm$0.43 &86.12$\pm$0.13 &68.14$\pm$0.53 \\
RA-GNN            &75.53$\pm$0.42 &93.74$\pm$0.53 &75.52$\pm$0.43 &58.43$\pm$0.63 &87.24$\pm$0.42 &58.32$\pm$0.64 &71.32$\pm$0.12 &87.42$\pm$0.13 &71.52$\pm$0.14 \\
GraphSMOTE        &74.32$\pm$0.23 &93.12$\pm$0.24 &73.92$\pm$0.23 &57.52$\pm$0.53 &86.73$\pm$0.42 &56.72$\pm$0.43 &68.23$\pm$0.32 &86.03$\pm$0.23 &68.23$\pm$0.32 \\
GraphENS          &75.82$\pm$0.43 &\underline{94.52$\pm$0.32} &74.13$\pm$0.32 &59.52$\pm$0.52 &87.53$\pm$0.42 &58.23$\pm$0.43 &72.23$\pm$0.23 &87.82$\pm$0.12 &71.92$\pm$0.23 \\
GraphMixup        &77.52$\pm$0.23 &94.32$\pm$0.42 &77.42$\pm$0.32 &\underline{63.32$\pm$0.52} &\underline{88.42$\pm$0.53} &\underline{63.12$\pm$0.73} &71.42$\pm$0.32 &86.62$\pm$0.23 &71.72$\pm$0.32 \\
CDCGAN            &\underline{78.12$\pm$0.32} &94.52$\pm$0.42 &\underline{77.92$\pm$0.43} &63.23$\pm$0.42 &84.23$\pm$0.42 &63.23$\pm$0.42 &\underline{72.82$\pm$0.42} &\underline{88.52$\pm$0.12} &\underline{74.42$\pm$0.52} \\
\midrule
\textbf{GraphSB}  &\textbf{81.90$\pm$0.32} &\textbf{96.43$\pm$0.32} &\textbf{81.78$\pm$0.23} &\textbf{64.42$\pm$0.32} &\textbf{90.82$\pm$0.42} &\textbf{64.82$\pm$0.42} &\textbf{75.82$\pm$0.23} &\textbf{89.52$\pm$0.32} &\textbf{74.72$\pm$0.42} \\
\bottomrule
\end{tabular*}
\end{table*}

\begin{table*}[!ht]
\centering
\small
\setlength{\tabcolsep}{4pt}
\caption{Performance comparison of different methods on BlogCatalog and Wiki-CS datasets.}
\label{tab:table2}
\begin{tabular*}{\textwidth}{@{\extracolsep{\fill}}lcccccc@{}}
\toprule
\multirow{2}{*}{Methods} & \multicolumn{3}{c}{BlogCatalog} & \multicolumn{3}{c}{Wiki-CS} \\
\cmidrule(lr){2-4} \cmidrule(lr){5-7}
                         & Acc & AUC & F1 & Acc & AUC & F1 \\
\midrule
Vanilla            & 20.82$\pm$0.51 &58.37$\pm$0.42 &6.73$\pm$0.23 &76.79$\pm$0.14 &94.05$\pm$0.25 &73.56$\pm$0.16 \\
Over-Sampling     &20.26$\pm$0.43 &59.21$\pm$0.34 &7.25$\pm$0.35 &77.94$\pm$0.26 &94.87$\pm$0.27 &74.43$\pm$0.28 \\
Re-weight         &20.44$\pm$0.55 &58.53$\pm$0.46 &6.98$\pm$0.27 &76.12$\pm$0.18 &93.98$\pm$0.29 &73.85$\pm$0.20 \\
SMOTE             &20.63$\pm$0.47 &59.59$\pm$0.38 &7.31$\pm$0.19 &78.07$\pm$0.31 &94.54$\pm$0.32 &74.58$\pm$0.32 \\
Embed-SMOTE       &20.22$\pm$0.61 &58.15$\pm$0.49 &7.02$\pm$0.21 &75.03$\pm$0.42 &94.32$\pm$0.33 &72.19$\pm$0.42 \\
DR-GNN            &24.49$\pm$0.45 &60.58$\pm$0.57 &11.93$\pm$0.46 &78.69$\pm$0.28 &95.03$\pm$0.19 &75.78$\pm$0.47 \\
RA-GNN            &25.32$\pm$0.37 &65.59$\pm$0.48 &12.47$\pm$0.49 &79.08$\pm$0.30 &95.26$\pm$0.22 &76.45$\pm$0.49 \\
GraphSMOTE        &24.75$\pm$0.58 &64.43$\pm$0.59 &12.38$\pm$0.22 &78.56$\pm$0.34 &95.54$\pm$0.24 &75.29$\pm$0.35 \\
GraphENS          &25.06$\pm$0.56 &65.72$\pm$0.47 &12.54$\pm$0.40 &79.27$\pm$0.18 &95.38$\pm$0.37 &76.16$\pm$0.43 \\
GraphMixup        &25.87$\pm$0.62 &\underline{66.53$\pm$0.13} &\underline{12.77$\pm$0.24} &\underline{80.45$\pm$0.26} &\underline{96.47$\pm$0.37} &\underline{76.59$\pm$0.25} \\
CDCGAN           &\underline{26.11$\pm$0.21} &66.45$\pm$0.36 &10.25$\pm$0.51 &79.28$\pm$0.40 &95.70$\pm$0.33 &76.40$\pm$0.28 \\
\midrule
\textbf{GraphSB} &\textbf{28.41$\pm$0.64} &\textbf{66.74$\pm$0.66} &\textbf{12.98$\pm$0.67} &\textbf{82.24$\pm$0.28} &\textbf{97.91$\pm$0.38} &\textbf{79.40$\pm$0.44} \\
\bottomrule
\end{tabular*}
\end{table*}

\subsection{Robustness to Different Imbalance Ratios}
\label{sec5.3}

To systematically evaluate the robustness of GraphSB under varying degrees of class imbalance, we conduct extensive experiments on the Cora dataset by adjusting the imbalance ratio $\rho \in \{0.1, 0.3, 0.5, 0.7, 0.9\}$. As shown in Figure~\ref{fig:imbalance ratios}, our model demonstrates strong performance in terms of both Accuracy and F1-score across different imbalance settings. Specifically, we observe that GraphSB exhibits remarkable adaptability to varying imbalance ratios, consistently achieving significant improvements across all settings. The F1-score results further validate our model's effectiveness and robustness, particularly in handling minority classes. More impressively, under extreme imbalance conditions ($\rho=0.1$), GraphSB achieves significant improvements in both metrics, outperforming the second-best method by 8\% in Accuracy  and 5.4\% in F1-score. Additionally, the ROC-AUC results presented in Table~\ref{tab:imbalance} further support these findings, where GraphSB achieves consistently superior performance with AUC scores ranging from 0.945 to 0.967, outperforming CDCGAN by 3.5\% under extreme imbalance ($\rho=0.1$) and maintaining this advantage across all imbalance ratios.

\begin{figure*}[!ht]
\centering
\includegraphics[width=\textwidth]{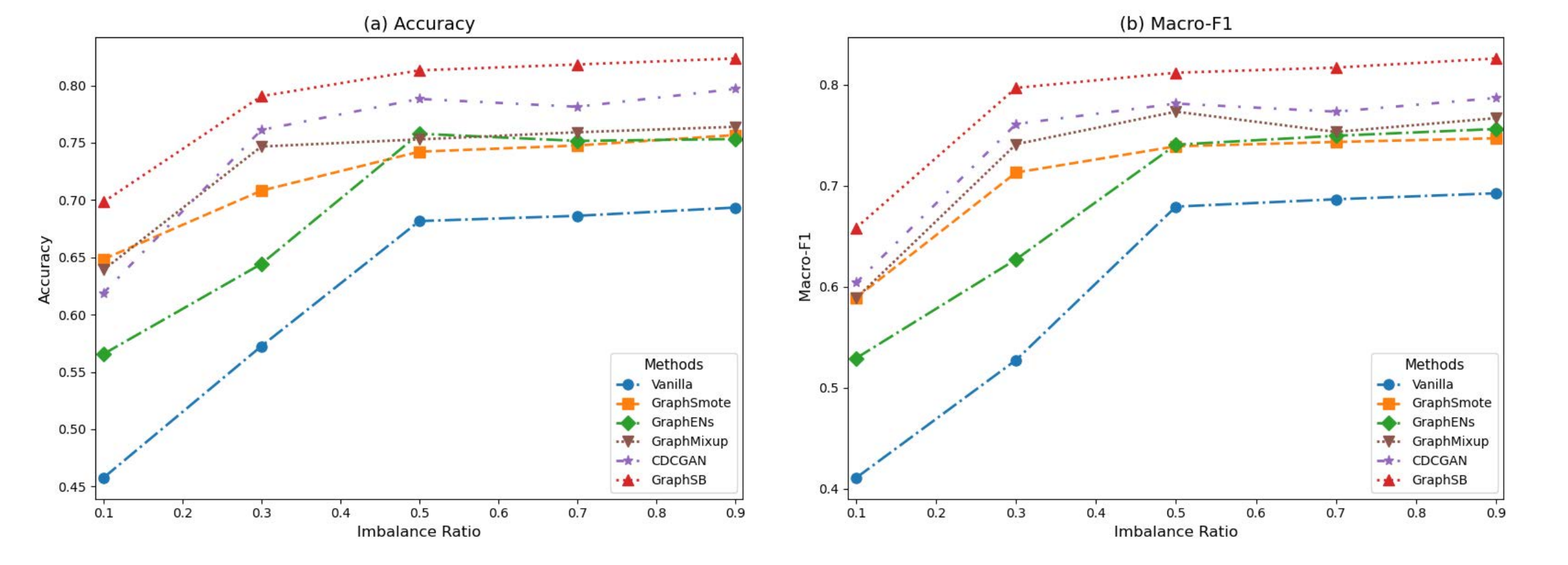}
\caption{Performance analysis with different imbalance ratios on Cora.}
\label{fig:imbalance ratios}
\end{figure*}

\begin{table*}[!ht]
\centering
\caption{Performance of different methods under varying class-imbalanced ratios $\rho$.}
\label{tab:imbalance}
\begin{tabular*}{\textwidth}{@{\extracolsep{\fill}}lcccccc}
\toprule
\textbf{Methods} & \multicolumn{5}{c}{\textbf{Class-Imbalanced Ratio $\rho$}} \\
\cmidrule(lr){2-6}
 & \textbf{0.1} & \textbf{0.3} & \textbf{0.5} & \textbf{0.7} & \textbf{0.9} \\
\midrule
Vanilla           & 0.843 & 0.907 & 0.919 & 0.923 & 0.925 \\
Over-Sampling    & 0.830 & 0.917 & 0.927 & 0.931 & 0.933 \\
Re-weight        & 0.869 & 0.921 & 0.925 & 0.929 & 0.931 \\
SMOTE            & 0.839 & 0.917 & 0.925 & 0.930 & 0.932 \\
Embed-SMOTE      & 0.870 & 0.906 & 0.918 & 0.927 & 0.929 \\
DR-GNN           & 0.890 & 0.921 & 0.929 & 0.937 & 0.939 \\
RA-GNN           & 0.895 & 0.925 & 0.933 & 0.941 & 0.943 \\
GraphSMOTE       & 0.887 & 0.923 & 0.930 & 0.935 & 0.937 \\
GraphENS         & 0.894 & 0.925 & 0.935 & 0.939 & 0.941 \\
GraphMixup       & 0.903 & \underline{0.931} & \underline{0.942} & 0.945 & 0.947 \\
CDCGAN           & \underline{0.910} & 0.926 & 0.940 & \underline{0.950} & \underline{0.949} \\
GraphSB          & \textbf{0.945} & \textbf{0.953} & \textbf{0.964} & \textbf{0.966} & \textbf{0.967} \\

\bottomrule
\end{tabular*}
\end{table*}

\begin{figure*}[!ht]
\centering
\includegraphics[width=\textwidth]{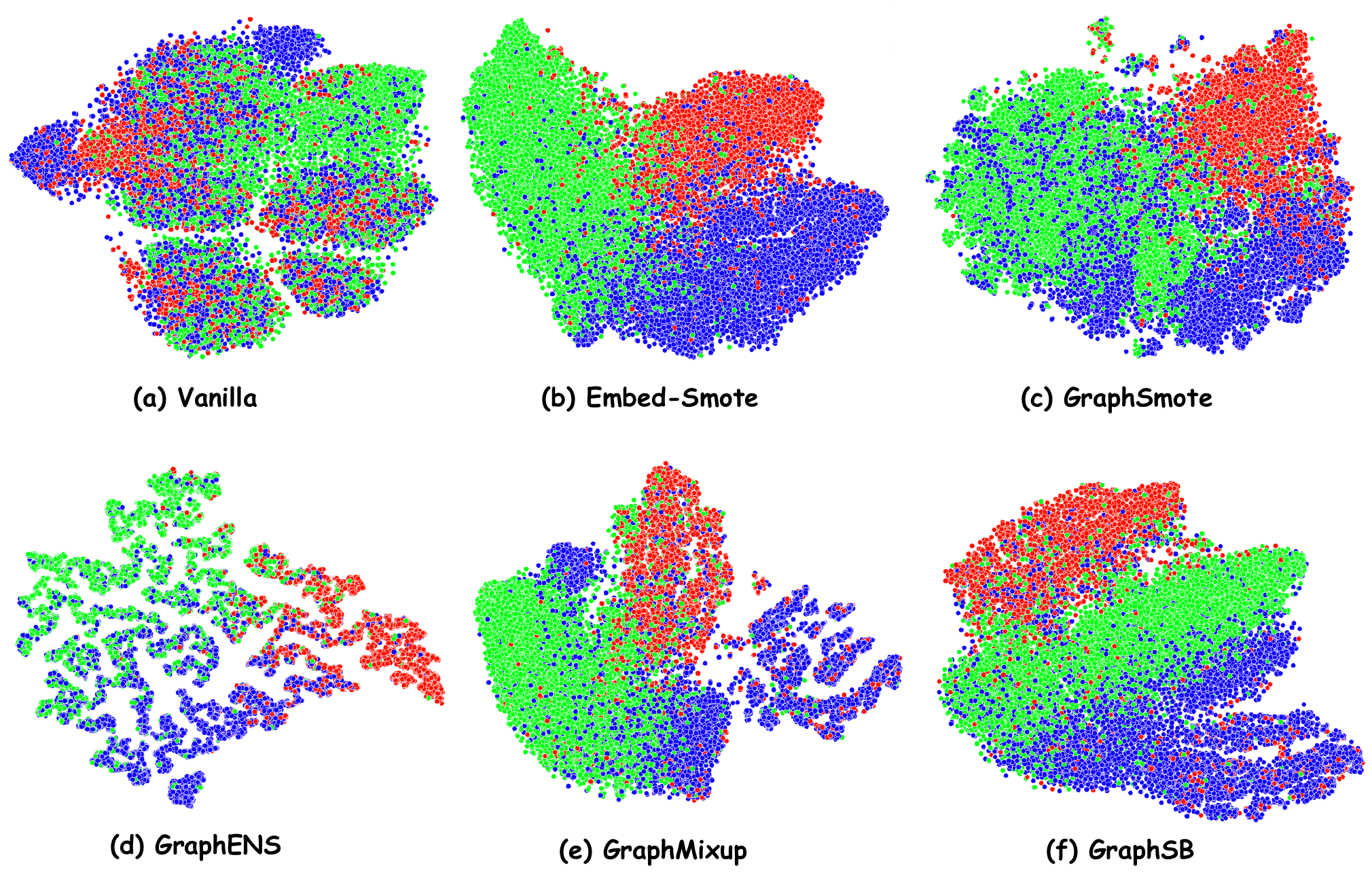}
\caption{Visualization on dataset PubMed.}
\label{fig4}
\end{figure*}

\subsection{Ablation study}
\label{sec5.4}
We develop a framework for structural balance-based enhancement, and to evaluate the impact of the two modules, Graph Structure Enhancement (SE) and Relation Diffusion (RD), on imbalanced node classification, we create the following variants of GraphSB based on these two modules: (A) Graph Structure Enhancement (w/o SE); (B) Relation Diffusion (w/o RD); (C) both Graph Structure Enhancement and Relation Diffusion (w/o Both), and (D) the full model.

\begin{figure*}[!ht]
\centering
\includegraphics[width=\textwidth]{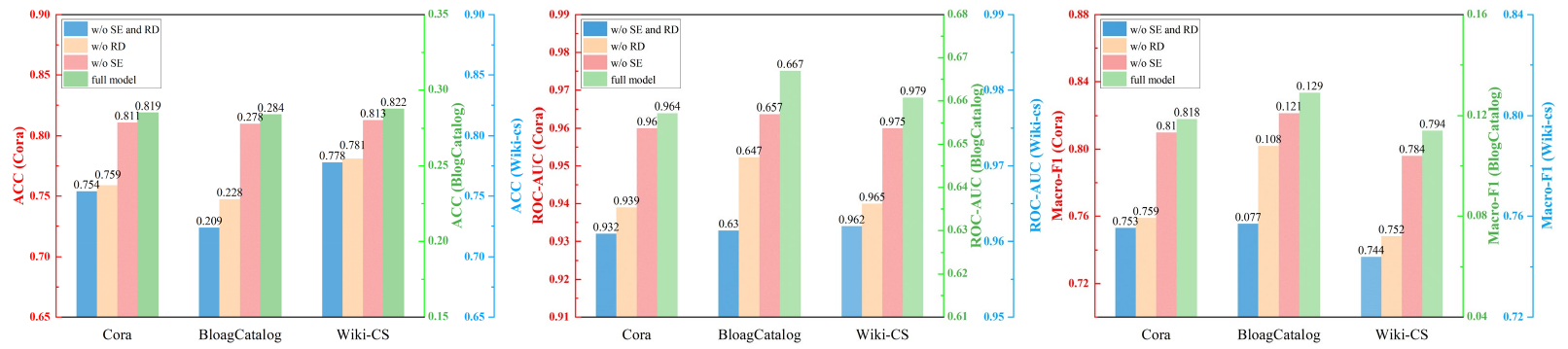}
\caption{Ablation study: The impact of each module.}
\label{Ablation study}
\end{figure*}

Figure~\ref{Ablation study} depicts the outcomes of ablation experiments across different datasets. From these results, we observe that both modules contribute significantly to the model's performance, with the complete GraphSB framework consistently outperforming its ablated variants. Removing both modules and using the default GCN network for node embedding updates is generally suboptimal because it tends to exacerbate the challenges faced by minority class nodes. On the other hand, Relation Diffusion increases the receptive field while the random edge dropout operation helps prevent overfitting; Structure Enhancement allows for minority enhancement, further improving node classification performance across various metrics.

\subsection{Over-Sampling Scale}
\label{sec5.5}

The reinforcement learning module addresses this through adaptive ratio control. For different datasets, GraphSB automatically determines appropriate over-sampling scales for different minority classes, starting from an initial ratio $\alpha_{\text{init}}^i$ for each dataset $i$. The final over-sampling scale is learned through $\alpha^i - \alpha_{\text{init}}^i$. Moreover, even on the same dataset, GraphSB can generate different numbers of supplementary nodes under different base architectures, as the over-sampling scale is learned adaptively. Figure~\ref{fig:rl_adaptive} demonstrates the over-sampling ratio adjustment process across all datasets.

\subsection{Visualization Analysis}
\label{sec5.6}

The visualization experiments were conducted to demonstrate the representation learning capability of GraphSB by projecting the learned 128-dimensional node embeddings into a 2-dimensional space via t-SNE on the PubMed dataset. The perplexity of t-SNE is set to 30 with 5000 iterations to ensure stable visualization results. The visualization results are illustrated in Figure~\ref{fig4}.We compare GraphSB with the Vanilla method and four state-of-the-art GNN-based methods specifically designed for imbalanced node classification. To ensure stable visualization results, the t-SNE perplexity is set to 30 with 5000 iterations. The visualization results reveal several significant findings: The vanilla GNN method exhibits substantial node entanglement across different classes, indicating its inherent limitations in handling class imbalance during representation learning. Although GraphSMOTE and GraphMixup demonstrate enhanced performance through their respective node augmentation strategies (SMOTE and Mixup), the boundaries between different classes remain ambiguous. GraphENS, as the strongest baseline method, achieves improved class separability through ego network generation and feature saliency filtering; however, the dispersed embedding of minority class nodes potentially compromises its classification performance. In contrast, GraphSB exhibits superior representation learning capability by successfully constructing three distinct and compact clusters for both majority and minority classes, demonstrating its effectiveness in handling class imbalance.

\begin{figure}[ht]
\centering
\includegraphics[width=0.5\textwidth]{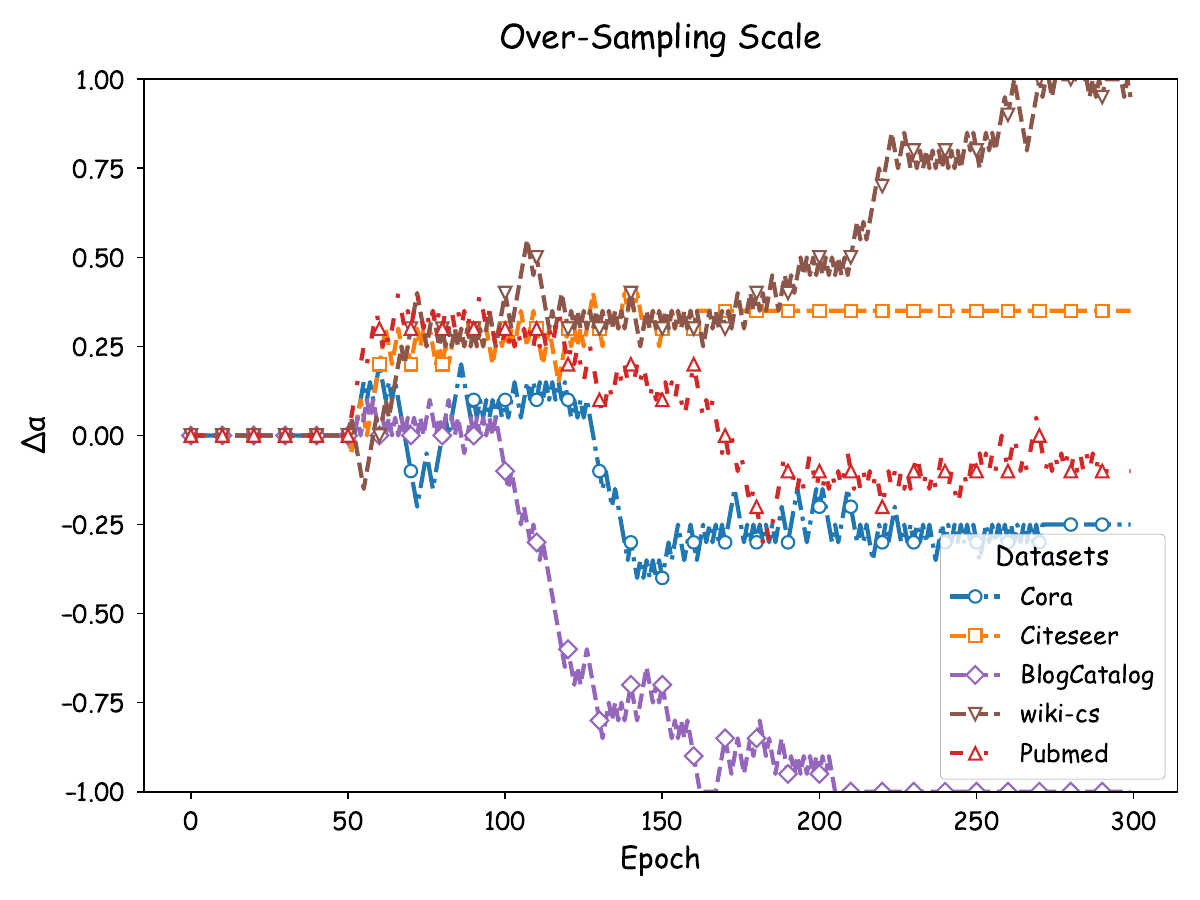}
\caption{Changes in over-sampling scale ($\alpha^i - \alpha_{\text{init}}^i$) of minority classes across all datasets.}
\label{fig:rl_adaptive}
\end{figure}

\begin{table*}[!ht]
\centering
\footnotesize
\setlength{\tabcolsep}{1pt}
\caption{Performance comparison of different methods with and without structural balance (SB) on Cora, Citeseer and BlogCatalog datasets.}
\label{tab:SB_results}
\begin{tabular*}{\textwidth}{@{\extracolsep{\fill}}lcccccccccc@{}}
\toprule
\multirow{2}{*}{Methods} & \multicolumn{3}{c}{Cora} & \multicolumn{3}{c}{Citeseer} & \multicolumn{3}{c}{BlogCatalog} \\
\cmidrule(lr){2-4} \cmidrule(lr){5-7} \cmidrule(lr){8-10}
                         & Acc & AUC & F1 & Acc & AUC & F1 & Acc & AUC & F1 \\
\midrule
Embed-SMOTE       & 72.3 $\pm$ 0.2 & 91.8 $\pm$ 0.2 & 72.2 $\pm$ 0.2 & 56.7 $\pm$ 0.6 & 83.0 $\pm$ 0.4 & 56.8 $\pm$ 0.6 & 20.22 $\pm$ 0.6 & 58.15 $\pm$ 0.5 & 7.02 $\pm$ 0.2 \\
Embed-SMOTE+SB    & 74.3 $\pm$ 0.4 & 93.6 $\pm$ 0.4 & 74.1 $\pm$ 0.4 & 67.8 $\pm$ 0.5 & 88.5 $\pm$ 0.4 & 67.7 $\pm$ 0.5 & 25.01 $\pm$ 0.4 & 63.41 $\pm$ 0.3 & 11.73 $\pm$ 0.3 \\
GraphSMOTE        & 74.3 $\pm$ 0.2 & 93.1 $\pm$ 0.2 & 73.9 $\pm$ 0.2 & 57.5 $\pm$ 0.5 & 86.7 $\pm$ 0.4 & 56.7 $\pm$ 0.4 & 24.75 $\pm$ 0.6 & 64.43 $\pm$ 0.6 & 12.38 $\pm$ 0.2 \\
GraphSMOTE+SB     & 81.0 $\pm$ 0.3 & 95.8 $\pm$ 0.2 & 80.2 $\pm$ 0.5 & 64.2 $\pm$ 0.4 & 88.7 $\pm$ 0.3 & 63.1 $\pm$ 0.4 & 27.09 $\pm$ 0.5 & 65.56 $\pm$ 0.4 & 12.51 $\pm$ 0.3 \\
GraphENS          & 75.8 $\pm$ 0.4 & 94.5 $\pm$ 0.3 & 74.1 $\pm$ 0.3 & 59.5 $\pm$ 0.5 & 87.5 $\pm$ 0.4 & 58.2 $\pm$ 0.4 & 25.06 $\pm$ 0.6 & 65.72 $\pm$ 0.5 & 12.54 $\pm$ 0.4 \\
GraphENS+SB       & 78.3 $\pm$ 0.3 & 95.5 $\pm$ 0.4 & 77.1 $\pm$ 0.4 & 62.8 $\pm$ 0.4 & 88.3 $\pm$ 0.4 & 61.5 $\pm$ 0.4 & 25.8 $\pm$ 0.4 & 66.43 $\pm$ 0.3 & 12.77 $\pm$ 0.3 \\
GraphMixup        & 77.5 $\pm$ 0.2 & 94.3 $\pm$ 0.4 & 77.4 $\pm$ 0.3 & 63.3 $\pm$ 0.5 & 88.4 $\pm$ 0.5 & 63.1 $\pm$ 0.7 & 25.87 $\pm$ 0.6 & 66.53 $\pm$ 0.1 & 12.77 $\pm$ 0.2 \\
GraphMixup+SB     & 81.9 $\pm$ 0.3 & 96.4 $\pm$ 0.3 & 81.7 $\pm$ 0.2 & 68.5 $\pm$ 0.5 & 88.9 $\pm$ 0.4 & 69.1 $\pm$ 0.5 & 28.41 $\pm$ 0.5 & 66.74 $\pm$ 0.4 & 12.98 $\pm$ 0.3 \\
CDCGAN            & 78.1 $\pm$ 0.3 & 94.5 $\pm$ 0.4 & 77.9 $\pm$ 0.4 & 64.2 $\pm$ 0.4 & 87.0 $\pm$ 0.4 & 64.0 $\pm$ 0.4 & 26.11 $\pm$ 0.2 & 66.45 $\pm$ 0.4 & 10.25 $\pm$ 0.5 \\
CDCGAN+SB         & 81.3 $\pm$ 0.3 & 95.8 $\pm$ 0.4 & 81.9 $\pm$ 0.3 & 65.2 $\pm$ 0.3 & 87.0 $\pm$ 0.4 & 64.2 $\pm$ 0.4 & 27.79 $\pm$ 0.4 & 64.42 $\pm$ 0.3 & 11.13 $\pm$ 0.4 \\
\bottomrule
\end{tabular*}
\end{table*}

\begin{table*}[!ht]
\centering
\small
\setlength{\tabcolsep}{2pt}
\caption{Performance comparison of different methods with and without structural balance (SB) on Cora, Citeseer and BlogCatalog datasets.}
\label{tab:rb_results_final}
\begin{tabular*}{\textwidth}{@{\extracolsep{\fill}}ll|cccccccccccc@{}}
\toprule
\multicolumn{2}{c}{Metric} & \multicolumn{2}{c}{Embed-SMOTE} & \multicolumn{2}{c}{GraphSMOTE} & \multicolumn{2}{c}{GraphENS} & \multicolumn{2}{c}{GraphMixup} & \multicolumn{2}{c}{CDCGAN}  & \multicolumn{2}{c}{AVG} \\
\cmidrule(lr){3-4} \cmidrule(lr){5-6} \cmidrule(lr){7-8} \cmidrule(lr){9-10} \cmidrule(lr){11-12} \cmidrule(lr){13-14}
\multicolumn{2}{c}{} & w/o SB & +SB & w/o SB & +SB & w/o SB & +SB & w/o SB & +SB & w/o SB & +SB & w/o SB & +SB\\
\midrule
\multirow{3}{*}{\rotatebox{90}{\makecell{Cora}}}

& Acc & 72.32 & 74.34 & 74.32 & 81.30 & 75.81 & 78.33 & 77.52 & 81.93 & 78.13 & 81.34 & 75.62& 79.45(+3.83) \\
& AUC & 91.82 & 93.64 & 93.12 & 95.85 & 94.53 & 95.54 & 94.34 & 96.43 & 94.54 & 95.84 & 93.67 & 95.46(+1.79)\\
& F1  & 72.22 & 74.14 & 73.92 & 80.27 & 74.16 & 77.18 & 77.44 & 81.78 & 77.94 & 81.93 & 75.14 & 79.06(+3.92)\\
\midrule
\multirow{3}{*}{\rotatebox{90}{\makecell{Citeseer}}}

& Acc & 56.74 & 67.86 & 57.55 & 64.24 & 59.55 & 62.84 & 63.31 & 68.58 & 64.21 & 65.25 & 60.27& 65.74(+5.47) \\
& AUC & 83.01 & 88.52 & 86.71 & 88.74 & 87.51 & 88.35 & 88.42 & 88.93 & 87.05 & 87.08 & 86.54 & 88.32(+1.78)\\
& F1  & 56.81 & 67.77 & 56.71 & 63.15 & 58.21 & 61.53 & 63.12 & 69.14 & 64.04 & 64.26 & 59.78 & 65.17(+5.39)\\
\midrule
\multirow{3}{*}{\rotatebox{90}{\makecell{Blog\\Catalog}}}
& Acc & 20.22 & 25.01 & 24.75 & 27.09 & 25.06 & 25.80 & 25.87 & 28.41 & 26.11 & 27.79  & 24.40 & 26.82(+2.42)\\
& AUC & 58.15 & 63.41 & 64.43 & 65.56 & 65.72 & 66.43 & 66.53 & 66.74 & 66.45 & 64.42  & 64.26 & 65.31(+1.05)\\
& F1  & 7.02  & 11.73 & 12.38 & 12.51 & 12.54 & 12.77 & 12.77 & 12.98 & 10.25 & 11.13  & 10.99 & 12.22(+1.23)\\
\bottomrule
\end{tabular*}
\end{table*}

\begin{figure*}[!ht]
\centering
\includegraphics[width=\textwidth]{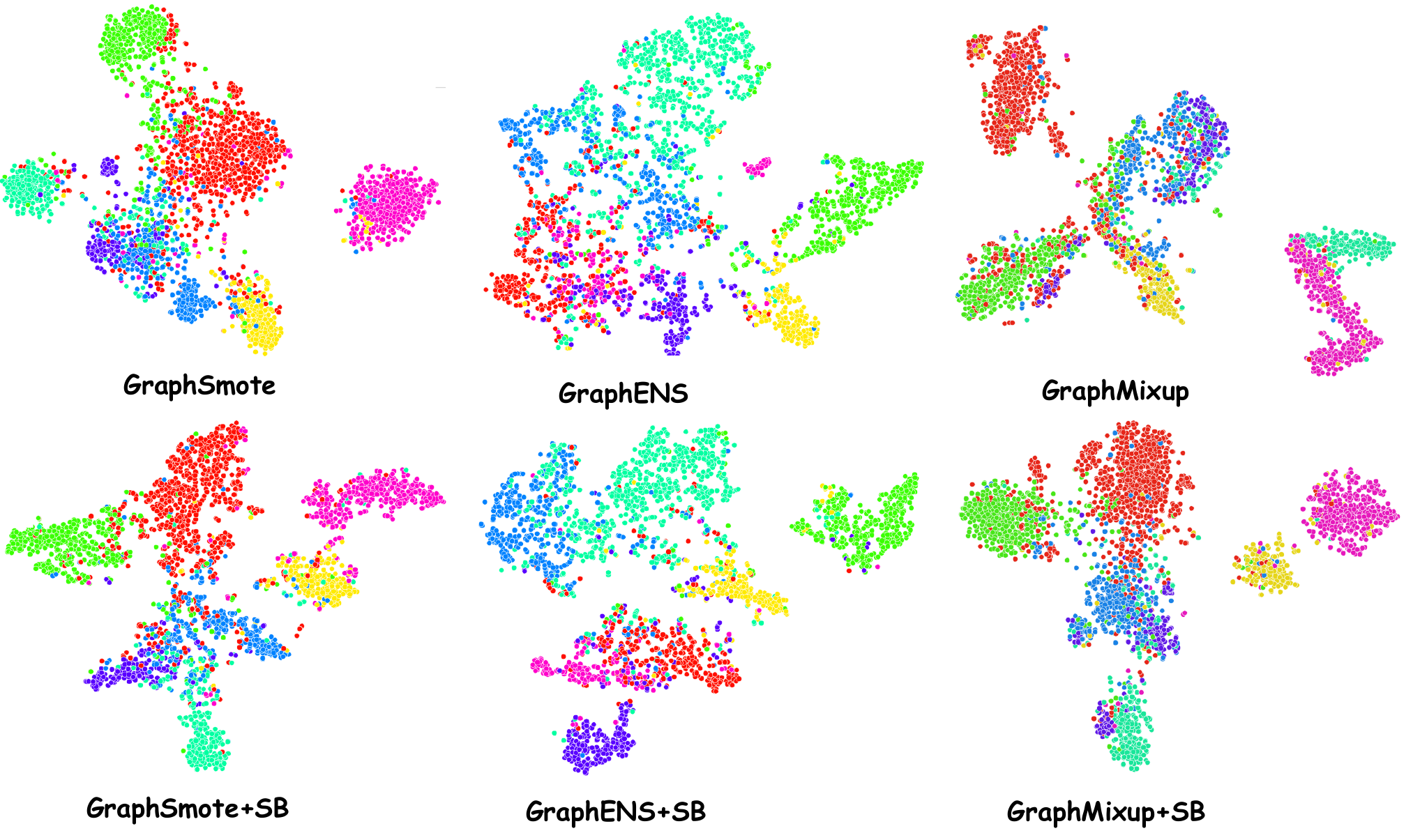}
\caption{Visualization on dataset Cora.}
\label{fig5}
\end{figure*}

\begin{figure}[!ht]
\centering
\includegraphics[width=0.5\textwidth]{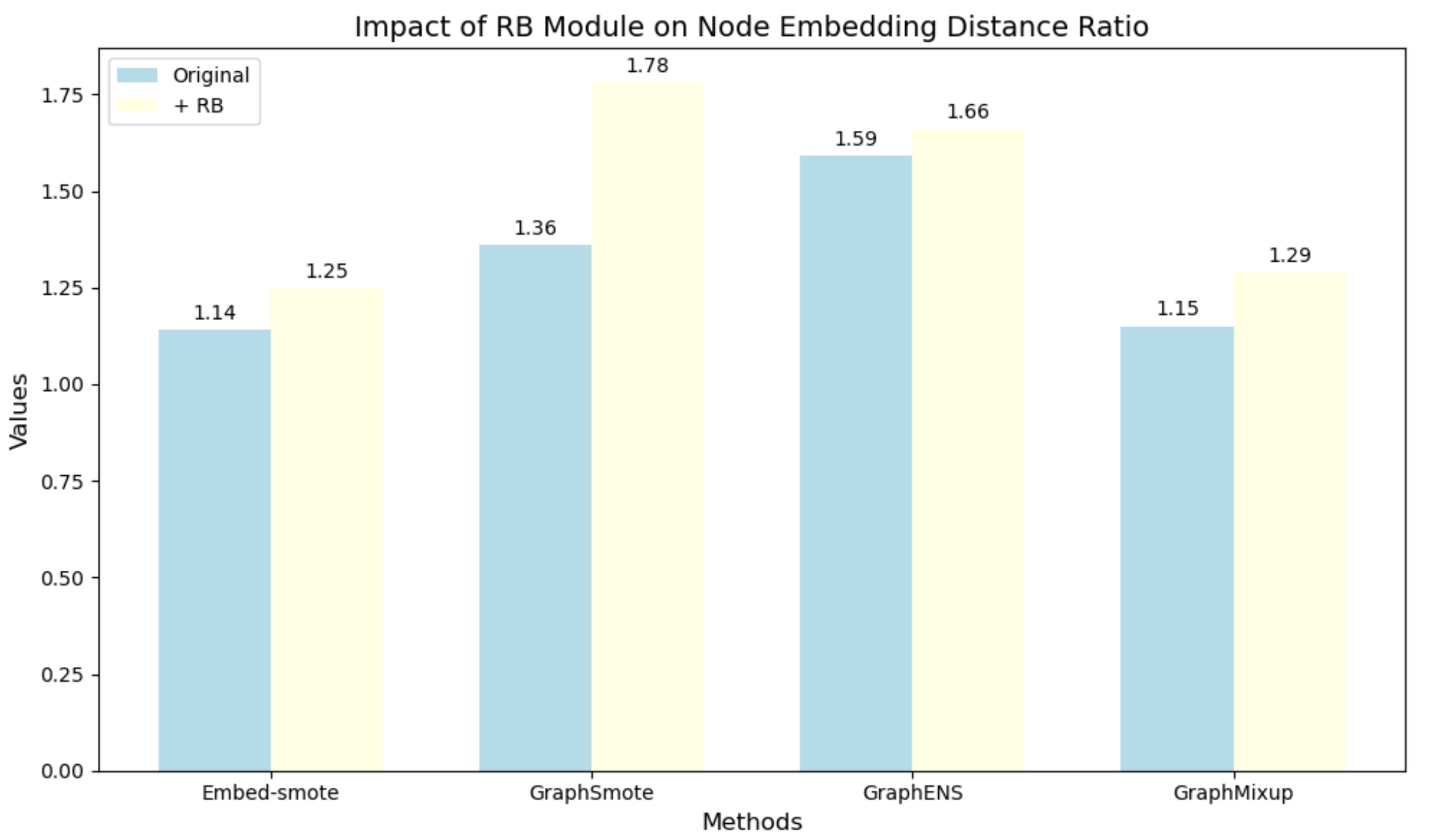}
\caption{Impact of SB Module on Node Embedding Distance Ratio.}
\label{fig2}
\end{figure}

\subsection{Universality of structural balance in Synthetic Networks}
\label{sec5.7}

Our proposed structural balance (SB) module can be seamlessly integrated into any synthetic method designed to address class imbalance in node classification with minimal code modifications. The core mechanism of SB involves initially concentrating attention on the target node, which subsequently diffuses to its neighboring nodes over multiple steps. This iterative diffusion process yields a continuous weighted diffusion matrix that defines the edge weights between the target node and other nodes. By prioritizing relational balance over numerical balance, the SB module enhances intra-class connectivity for minority nodes while mitigating cross-class interference from majority nodes.  

To evaluate the quality of the learned graph embeddings, we introduce the intra-to-inter-class distance ratio as a novel metric to quantify both intra-class consistency and inter-class separability. For a category \( C_k \) containing \( N_k \) samples, the intra-class distance is defined as:  
\[
D_{\text{intra}}(C_k) = \frac{1}{N_k (N_k - 1)} \sum_{i \neq j} d(x_i, x_j),
\]  
where \( d(x_i, x_j) \) represents the Euclidean distance between the embeddings of nodes \( x_i \) and \( x_j \) within the same category. The inter-class distance between two distinct categories \( C_i \) and \( C_j \) is defined as:  
\[
D_{\text{inter}}(C_i, C_j) = \frac{1}{N_i N_j} \sum_{x_i \in C_i} \sum_{x_j \in C_j} d(x_i, x_j).
\]  
We define the intra-to-inter-class distance ratio as:
\begin{equation}
R = \frac{\frac{1}{|C|} \sum_{k} D_{\text{inter}}(C_k, C_{\bar{k}})}{\frac{1}{|C|} \sum_{k} D_{\text{intra}}(C_k)},
\label{eq:ratio}
\end{equation}
where \( C_{\bar{k}} \) represents the set of all categories except \( C_k \). A higher value of \( R \) indicates that the embeddings exhibit greater inter-class separation relative to intra-class compactness, which is favorable for improving downstream classification performance.

Table~\ref{tab:SB_results} shows a performance comparison on the Cora dataset between various baseline methods and their SB-enhanced counterparts. In addition, Figure~\ref{fig5} presents the visualization on Cora, and the impact of the SB module on node embedding distance ratio is depicted in Figure~\ref{fig2}.


\section{Conclusion}
\label{sec6}

In the field of graph data, the challenges of handling imbalanced node classification stem not only from the imbalance in the number of categories but also from the complexity of the graph structure. Specifically, as graphs are non-Euclidean data, the balance of relationships between nodes is equally important. For instance, although the Cora dataset might exhibit good balance in terms of node relationships, intentionally creating an imbalance in category numbers can reveal different behaviors in models, which is beneficial in practical applications as it helps to reduce the model's dependence on the volume of data. In scenarios where node relationships are relatively balanced but the number of nodes is not, it is crucial to consider information from non-direct neighbors when updating node embeddings. Furthermore, when both node relationships and numbers are imbalanced, establishing reasonable inter-node relationships and considering non-direct neighbors during the embedding update process can enhance the embedding quality of minority class nodes and effectively reduce biases introduced during the data synthesis process. This comprehensive approach to considering both category imbalance and structural imbalance can significantly improve the performance and fairness of imbalanced node classification tasks.

\bibliography{sn-bibliography}
\end{document}